\documentclass[10pt,twocolumn,letterpaper]{article}

\usepackage{cvpr}
\usepackage{float}
\usepackage{url}
\usepackage{enumitem}
\usepackage{amsmath,amsfonts,amssymb,amsthm}
\usepackage{algorithmic}
\usepackage{graphicx}
\usepackage{textcomp}
\usepackage{xcolor}
\usepackage{xurl}
\usepackage[utf8]{inputenc}
\usepackage{booktabs}
\usepackage{tabularx}
\usepackage[accsupp]{axessibility}

\usepackage{array}
\newtheorem{theorem}{Theorem}

\newtheorem{corollary}[theorem]{Corollary}

\newtheorem{definition}{Definition}

\newcommand{\Prv}[1]{\operatorname{Pr}\!\left[#1\right]}
\newcommand{\RR}{\mathcal{M}_{\mathrm{RR}}}
\newcommand{\Adv}{\mathsf{Adv}}

\definecolor{cvprblue}{rgb}{0.21,0.49,0.74}
\usepackage[pagebackref,breaklinks,colorlinks,allcolors=cvprblue]{hyperref}

\def\paperID{6142} 
\def\confName{CVPR}
\def\confYear{2026}

\title{LDP-Slicing: Local Differential Privacy for Images via Randomized Bit-Plane Slicing}

\author{Yuanming Cao \quad Chengqi Li \quad Wenbo He \\
McMaster University\\
Hamilton, Ontario, Canada\\
 {\tt\small \{caoy15, lic222, hew11\}@mcmaster.ca}
 }
\begin{document}
\maketitle
\begin{abstract}
Local Differential Privacy (LDP) is the gold standard trust model for privacy-preserving machine learning by guaranteeing privacy at the data source. However, its application to image data has long been considered impractical due to the high dimensionality of pixel space. Canonical LDP mechanisms are designed for low-dimensional data, resulting in severe utility degradation when applied to high-dimensional pixel spaces. This paper demonstrates that this utility loss is not inherent to LDP, but from its application to an inappropriate data representation. We introduce LDP-Slicing, a lightweight, training-free framework that resolves this domain mismatch. Our key insight is to decompose pixel values into a sequence of binary bit-planes. This transformation allows us to apply the LDP mechanism directly to the bit-level representation. To further strengthen privacy and preserve utility, we integrate a perceptual obfuscation module that mitigates human-perceivable leakage and an optimization-based privacy budget allocation strategy. This pipeline satisfies rigorous pixel-level $\varepsilon$-LDP while producing images that retain high utility for downstream tasks. Extensive experiments on face recognition and image classification demonstrate that LDP-Slicing outperforms existing DP/LDP baselines under comparable privacy budgets, with negligible computational overhead. 
\end{abstract}
    
\section{Introduction}
The architecture of modern vision systems is built on a fundamental, high-stakes trade: to unlock the power of services like biometric authentication or AI-assisted medical diagnostics, users must surrender their privacy-sensitive, identifiable image to a central server. Even with the best security practices, this centralization is vulnerable to internal breaches and unauthorized misuse. Recurring violations of data privacy regulations, such as the collection of facial templates without user consent \cite{patel2015_facebook_bipa}, confirm this systemic misuse. The urgent question for the field is how to preserve image privacy and enable powerful vision models in a zero-trust world.

\begin{figure*}[!t]
  \centering
  \includegraphics[width=\textwidth]{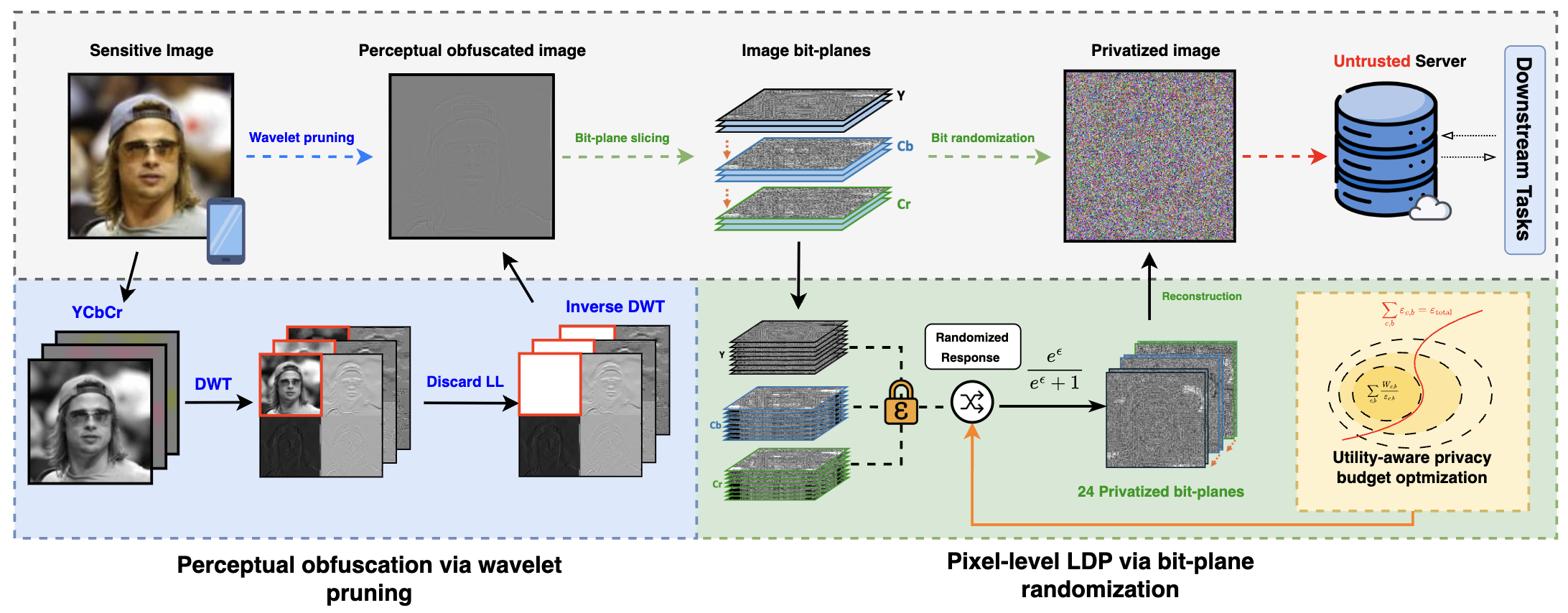}
  \caption{\textbf{The LDP-Slicing framework.} Our method consists of two primary stages: (1) Perceptual obfuscation: the input image is transformed into the frequency domain via DWT, where the low-frequency (LL) band is pruned to remove human-perceptible information. (2) Bit-plane randomization: The obfuscated image is decomposed into binary bit-planes. A utility-aware randomized response mechanism is applied to each bit and enforces a strict $\varepsilon$-Local Differential Privacy guarantee before the final image is reconstructed.}
  \label{fig:pipline}
\end{figure*}

Early approaches to image privacy relied mainly on simple visual obfuscation, such as blurring and pixelation \cite{35481,blur}. Although visually effective against human observers, these methods offer no formal privacy guarantees and have been shown to be reversible by modern deep learning attacks \cite{101145}. Meanwhile, cryptographic techniques \cite{pmlrv48giladbachrach16,infocom} provide mathematically strong security guarantees against unauthorized access, but cannot prevent inference attacks and their computational demands make them impractical for training large-scale vision models. 

Differential Privacy (DP) \cite{DworkMNS06,DworkMNS07} has emerged as the gold standard algorithm to provide rigorous, quantifiable privacy guarantees. In centralized DP, data is first collected by a central curator who injects noise into query results or trained models \cite{Abadi_2016}. This setting benefits from higher utility but is built on an unrealistic assumption of a trusted curator. If the curator is breached, the privacy of all raw data is irreversibly lost, violating the core principle of a zero-trust world.

To eliminate the reliance on a trusted curator, we turn to Local Differential Privacy (LDP) \cite{kasiviswanathan2010learnprivately}, which enforces privacy at the data source. However, LDP confronts an insurmountable barrier with high-dimensional data. A single 8-bit pixel can assume 256 distinct values. Applying the LDP mechanism like randomized response in this high-cardinality space injects overwhelming noise that destroys nearly all task-relevant information and leads to the well-known ``curse of dimensionality" \cite{duchi2014localprivacydataprocessing}.

As a result, works for provable vision privacy have generally followed two alternative paths: either operating within the centralized DP model \cite{zhu2020private, golatkar2022mixed}, or applying LDP to a lower-dimensional representation, such as latent embeddings \cite{li2021differentially}, feature descriptors \cite{pittaluga2023ldp}, or eigenvectors \cite{peep}.

We argue this utility loss is not a limitation of LDP, but a solvable problem of domain mismatch. Our central premise is that an image can be transformed into an LDP-friendly representation without collapsing its dimensionality. This reframes the catastrophic utility loss as a consequence of applying LDP to an improper data representation. To achieve this, we introduce LDP-Slicing, a training-free framework that makes LDP practical for standard images. The framework first applies a wavelet-based perceptual obfuscation to defend against human inspection, then maps each obfuscated image to an LDP-friendly binary bit-plane space and applies a utility-aware randomized response at the bit level. The resulting privatized images satisfy rigorous per-pixel $\varepsilon$-LDP and plug directly into standard recognition/classification models. Extensive experiments show that LDP-Slicing delivers strong privacy–utility trade-offs with negligible overhead.

Our main contributions are as follows:
\begin{itemize}
  \item We introduce LDP-Slicing, a practical framework that, for the first time, enables rigorous pixel-level $\varepsilon$-LDP for standard images, without relying on handcrafted features or learned representations. The privatized image is fully compatible with the standard vision pipelines without requiring architectural changes.
  \item We propose an optimization-driven budget allocation strategy that distributes privacy noise according to the structural and perceptual importance of each bit-plane and channel.  
  \item We provide a formal proof that LDP-Slicing satisfies rigorous, pixel-level $\varepsilon$-Local Differential Privacy, and empirically validate its resilience against an identity distinguishing attack.
  \item We demonstrate state-of-the-art privacy-utility tradeoffs on four face recognition and two image classification benchmarks. Our method is lightweight and efficient, making it well-suited for on-device deployment.
\end{itemize}

\section{Related work}
\subsection{Image obfuscation}
The most direct approach to image privacy is to obscure visually sensitive content. Traditional methods like blurring, pixelation, and noise injection \cite{pxil, blur}, while simple to implement, offer no formal guarantees and have been repeatedly shown to be vulnerable to reconstruction by modern machine learning attacks \cite{zhang2022deepimagedeblurringsurvey}. A more sophisticated variant, learnable instance encoding, exemplified by InstaHide \cite{huang2021}, mixes sensitive images with a large public dataset and randomly flips the sign of pixels to conceal visual information. However, like other purely obfuscation-based methods, it offers no formal privacy guarantee and indeed was decisively broken by Carlini et al. \cite{carlini2021}, who demonstrated near-perfect reconstruction attacks, and subsequently by fusion-denoising attacks \cite{Luo_2022}. The consistent failure of these methods reveals a core principle: visual obscurity is a fragile defense against increasingly powerful learning-based attacks.

A separate line of work attempts to achieve privacy preservation in the frequency domain. These methods transform images using the Discrete Cosine Transform (DCT) \cite{Chen1977AFC} and then perturb or discard specific frequency coefficients \cite{ppfr-fd, duetface}. While removing low-frequency components can degrade visual information and preserve utility for some tasks \cite{wang2020highfrequencycomponenthelps}, this frequency removal method is also reversible \cite{Li_2010}. In stark contrast to these heuristic approaches, LDP-Slicing provides a formal, mathematically provable guarantee of privacy that is resilient by design. 
\subsection{Differential Privacy for images}
 To achieve provable privacy guarantees, most research has adopted DP with a centralized trust model. In this setting, a trusted curator collects the dataset of sensitive images and applies a DP mechanism, by injecting calibrated noise into the gradients during model training (e.g., DP-SGD \cite{Abadi_2016}). This has been successfully used to train private generative models like DP-GANs \cite{xie2018differentiallyprivategenerativeadversarial} or, more recently, DP-Diffusion models \cite{298240}, which can synthesize private images. While LDP removes the need for a trusted curator, it faces the well-known challenge with high-dimensional images, often referred to as the curse of dimensionality. Consequently, prior work avoids applying LDP to raw pixels, instead applying LDP to lower dimensional intermediate representations, such as feature vectors \cite{peep, pittaluga2023ldp}. This leaves a critical gap in the literature, which our work directly addresses: a method that provides the source-level privacy guarantees directly on high-dimensional pixel data without collapsing it into a low-dimensional representation.
\section{Preliminaries}
\subsection{Threat model}\label{sec:thret}
 We consider a powerful adversary who controls the server and has full knowledge of our privacy mechanism, but cannot access raw user data. The adversary's goal is to perform either a reconstruction attack (recovering a perceptually similar image $I^*$ from a privatized image $\tilde{I}$) or an identity distinguishing attack (determining if a privatized image $\tilde{I}$ and a public image $I_{\text{pub}}$ belong to the same person). We adopt the formal definitions for the reconstruction attack from Carlini et al. \cite{carlini2021} and provide the full details in supplementary. 
 \begin{definition}[Identity distinguishing attack]
\label{def:identity-disthin}
Let $\mathcal{M} : \mathcal{X} \to \widetilde{\mathcal{X}}$ be a randomized local privacy mechanism over face images. Let $\mathcal{I}$ be the set of identities, and for every $i \in \mathcal{I}$ let $D_i^{\mathrm{priv}}$ and $D_i^{\mathrm{pub}}$ be the private and public image distributions of that identity. The identity distinguishing experiment samples an identity $i$, draws $x \sim D_i^{\mathrm{priv}}$, releases $\tilde{x} = \mathcal{M}(x)$, and then, with probability $1/2$, draws $y \sim D_i^{\mathrm{pub}}$ (the match case) and otherwise draws $j \neq i$ and $y \sim D_j^{\mathrm{pub}}$ (the non-match case). The bit $b = \mathbf{1}[i = j]$ is the ground truth. An adversary $\mathcal{A}$ outputs $b' = \mathcal{A}(\tilde{x}, y)$, and its advantage is
\[
  \Adv^{\mathrm{link}}_{\mathcal{M}}(\mathcal{A})
  = \bigl| \Pr[b' = b] - \tfrac{1}{2} \bigr| .
\]
The probability is over the sampling of identities/images, the randomness of $\mathcal{M}$, and the randomness of $\mathcal{A}$.
\end{definition}
\subsection{Local Differential Privacy}
 Our work is grounded in the formal guarantees of Local Differential Privacy:
\begin{definition}[$\varepsilon$-Local Differential Privacy \cite{kasiviswanathan2010learnprivately}]
\label{def:dp}
A randomized mechanism $\mathcal{M}\!:\!\mathcal{X}\!\to\!\mathcal{Y}$ is
\emph{$\varepsilon$-LDP} if for all $x,x'\!\in\!\mathcal{X}$ and all measurable subset
$S\subseteq\mathcal{Y}$,
\[
  \Prv{\mathcal{M}(x)\in S}\le
e^{\varepsilon}\,\Prv{\mathcal{M}(x')\in S}.\]
\end{definition}
Unlike centralized DP, which ensures privacy by guaranteeing indistinguishability between neighboring datasets, LDP does not require defining neighboring dataset and provides a stronger privacy guarantee by ensuring that for any two local inputs, the mechanism’s output distributions are nearly identical. The privacy budget $\varepsilon$ controls the privacy-utility trade-off (\eg, smaller values imply stronger privacy and greater utility loss).
A canonical mechanism for achieving LDP in binary data is the randomized response.
\begin{definition}[Randomized response~\cite{randomrespnse}]
\label{def:rr}
Let $v$ take values in the binary domain $v \in \{0,1\}$, the (Binary) Randomized
Response $M_{RR}$ mechanism that takes $v$ as input and outputs a perturbed
$v' \in \{0,1\}$ can be defined as:
\[
\Prv{\RR(v)=v'} \;=\;
\begin{cases}
p      & \text{if } v = v',\\[2pt]
1-p    & \text{if } v \ne v'.
\end{cases}
\]
\end{definition}
The input $v$ can be viewed as the true input, and the probability of a truthful response is denoted as $p$. Such a mechanism has two nice properties:
\begin{theorem}[Sequential composition {\cite{DworkMNS06}}]
Let $\mathcal{M}_1,\mathcal{M}_2$ be randomized algorithms that are
$\varepsilon_1$- and $\varepsilon_2$‑LDP, respectively.  
Then their release \(\mathcal{M}_{1,2}(x)\!=\!(\mathcal{M}_1(x),\mathcal{M}_2(x))\)
is $(\varepsilon_1+\varepsilon_2)$‑LDP.
\label{thrm:seqcom}
\end{theorem}
\begin{theorem}[Closedness to postprocessing \cite{DworkMNS07}]
Let $\mathcal{M} : \mathbb{N}^{|\mathcal{X}|} \to R$ be a randomized algorithm that is $\varepsilon$-local differentially private. Let $f : R \to R'$ be an arbitrary randomized mapping. Then $f \circ \mathcal{M} : \mathcal{X} \to R'$ is $\varepsilon$-local differentially private.
\label{thrm:postpo}
\end{theorem}

\section{Methodology}
Our goal is to design a local privacy mechanism $\mathcal{M}$ that transforms an image $I$ into a privatized version $\tilde{I}$ such that $\tilde{I}$ satisfies per-pixel $\varepsilon$-LDP and remains useful for downstream vision tasks. At its core, the LDP mechanism provides privacy by injecting calibrated noise into the data source and makes it statistically indistinguishable from original input. The challenge is that the effectiveness of this noise injection is proportional to the size of the input space.

Consider a $k$-ary LDP mechanism (\eg, general randomized response \cite{randomrespnse}),  the mechanism must obfuscate the true value among all $k-1$ alternatives. The probability of reporting the true value is $p = \frac{e^{\varepsilon}}{e^{\varepsilon}+k-1}$. For an image with 8-bit pixels, the cardinality is extremely large ($k=256$). Even for a moderate privacy budget, the probability of reporting the true pixel value becomes vanishingly small, forcing the output to be almost pure noise.

We argue that this utility loss is not an unavoidable cost of LDP and can be solved with correct data representation. Our core insight is this: A pixel with 256 discrete states is a binary encoding of eight bits, and the significance of each bit corresponds to the semantic structure of the pixel. This is the foundation of LDP-Slicing (\cref{fig:pipline}), where we can allocate privacy budget based on the significance of the bit. Our framework is composed of three principled modules: first, a perceptual obfuscation stage to defeat human inspection. Second, our core bit-plane randomization mechanism provides a formal LDP guarantee against algorithmic attacks. Finally, a utility-aware optimization to allocate the privacy budget for maximum downstream utility.
\subsection{Perceptual obfuscation via wavelet pruning}
To provide a holistic defense, we begin by addressing the threat of direct human inspection. While our core LDP mechanism protects against algorithmic inference, at weaker privacy settings (\eg, larger $\varepsilon$), the noisy image may still retain residual structure perceptible to a human observer. To address this, we introduce a perceptual obfuscation as a pre-processing step.
This stage is motivated by a well-studied asymmetry in perception: human vision primarily uses low-frequency image cues (\eg, broad shapes and smooth regions) to recognize content~\cite{hvs}, whereas modern convolutional neural networks (CNNs) often exploit fine, high-frequency details~\cite{abello2021dissecting}. We exploit this by using a 1-level Haar Discrete Wavelet Transform (DWT) to decompose each image channel into a low-frequency approximation sub-band (LL) and three high-frequency detail sub-bands (LH, HL, HH). We then perform LL-Pruning by setting all coefficients in the low-frequency LL band to zero before reconstructing the image via the inverse DWT (IDWT). The effect of this process is visualized in \cref{fig:five_in_row} and \cref{fig:aba_dc}. Unlike the block-based Discrete Cosine Transform (DCT) and other frequency transformations, the DWT decomposes the image at multiple resolutions, and avoids introducing the artifacts common to DCT when frequency coefficients are aggressively pruned. This leads to better preservation of useful high-frequency details, as we empirically validate it in our ablation study \cref{tab:abla}.
\begin{figure}[t]                     
  \centering
  \includegraphics[width=0.9\columnwidth]{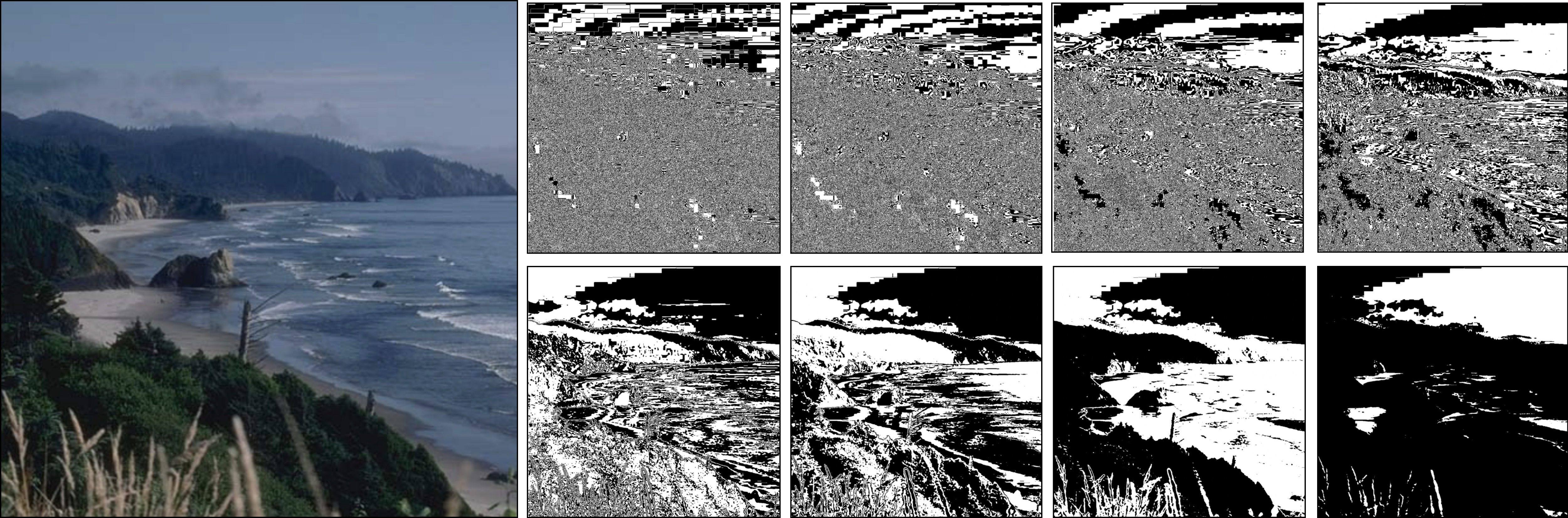}
  \caption{\textbf{Bit-plane slicing reveals the non-uniform distribution of structural information.} An 8-bit image (left) is decomposed into its planes (\textbf{right}), from LSB (top-left) to MSB (bottom-right). This visualization shows that coarse structural information is concentrated in the high-order MSB planes, while low-order LSB planes consist primarily of noise-like texture. This motivates our non-uniform, utility-aware budget optimization strategy.
  }
  \label{fig:bitplane}
\end{figure}
\begin{figure}[t]                     
  \centering
  \includegraphics[width=0.9\columnwidth]{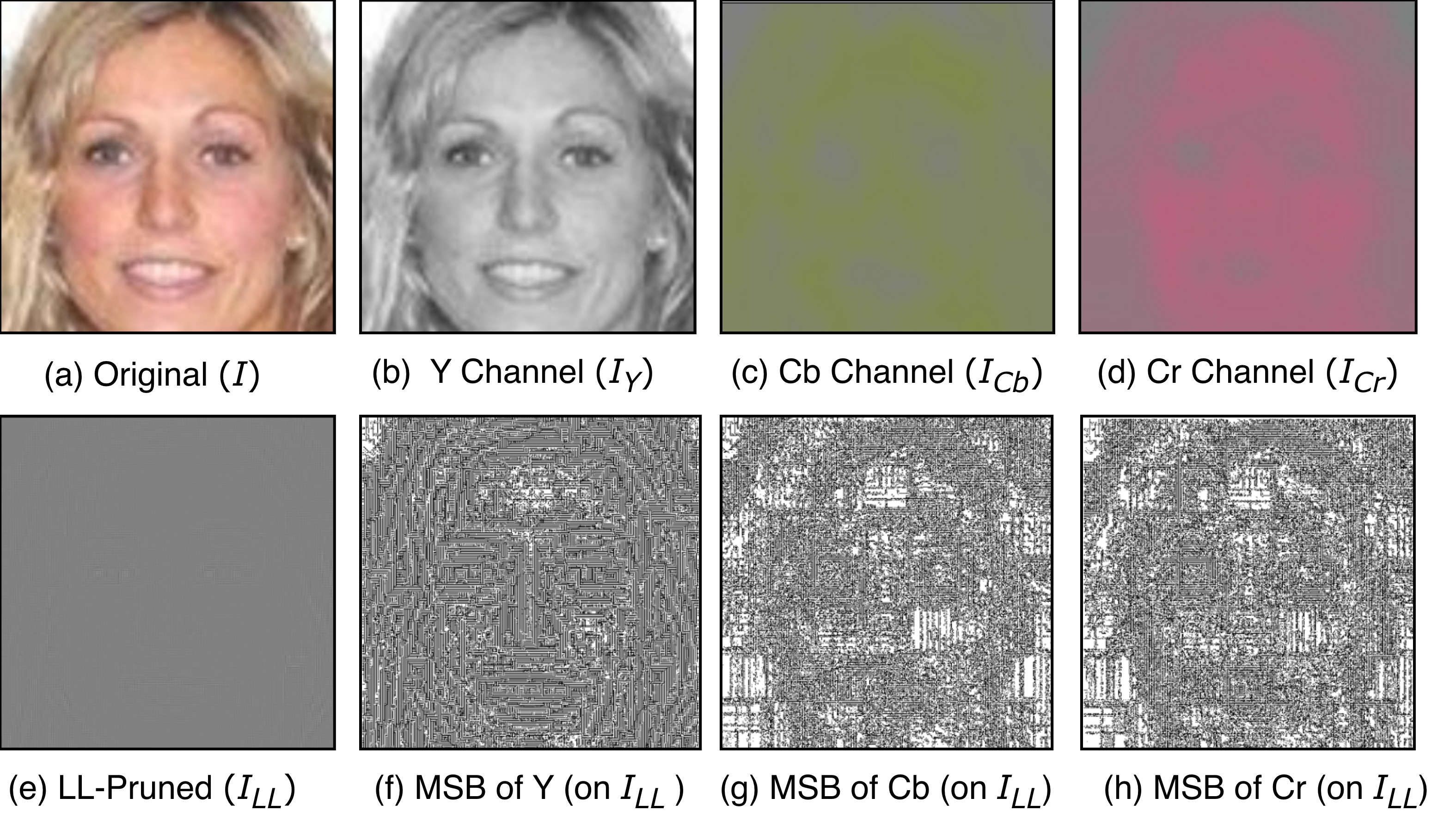}
  \caption{\textbf{Visual justification for channel-aware allocation and perceptual obfuscation.} This figure shows an original image (a) and its YCbCr components before and after LL-pruning. It validates two key decisions: (1) Channel Importance: The luma (Y) channel (b) contains the most structural information compared to the chroma channels (c,d). (2) Post-Pruning Signal: After LL-pruning (e), the resulting image (e) is perceptually obfuscated for human viewers. The Y-channel's MSB (f) still retains significant high-frequency detail for machine learning, unlike the less informative chroma MSBs (g,h).
  }
  \label{fig:five_in_row}
\end{figure}
\subsection{Pixel-level LDP via bit-plane randomization}
With the image perceptually obfuscated, we now introduce our core mechanism to provide a formal $\varepsilon$-LDP guarantee. We achieve this by first changing the data representation through Bit-Plane Slicing (BPS). 

For an image with $d$-bit pixel (typically $d=8$), each pixel $x \in [0, 2^d-1]$ is decomposed into a sequence of binary bits $\{x_1,...,x_d\}$, where:
\begin{equation}
x_{\ell} \;=\;
\Bigl\lfloor \frac{x}{2^{d-\ell}} \Bigr\rfloor \bmod 2,
\quad \ell \in \{1,\dots,d\}.
\end{equation}  
This process decomposes the image into a total of 24 bit planes across Y, Cb, and Cr channels (\cref{fig:bitplane}). In this binary domain, we can now apply the randomized response mechanism independently to flip each bit \( x_\ell \in \{0, 1\} \). This produces a privatized bit \( \tilde{x}_\ell \) according to: 
\begin{equation}
\Pr\!\bigl[\mathcal{M}_{RR}(x_\ell)=\tilde{x}_\ell\bigr]
   =\;       
   \begin{cases} 
\frac{e^{\varepsilon_{\ell}}}{\,e^{\varepsilon_{\ell}} + 1\,}.     & \text{if } \tilde{x}_\ell = x_\ell,\\
\frac{1}{\,e^{\varepsilon_{\ell}} + 1\,}.    & \text{otherwise}.
\end{cases}
\end{equation}
where $\varepsilon_{\ell}$ is the privacy budget allocated to that specific bit-plane. The question then arises: how should the total privacy budget $\varepsilon_\text{total}$ be distributed among the individual bit-plane budgets $\{\varepsilon_\ell\}$?
\subsection{Utility-aware privacy budget optimization}\label{sec:optmization}
A naive application of randomized response would allocate the total privacy budget $\varepsilon_{\text{total}}$ uniformly across all 24 bit-planes (\eg, $\varepsilon_{\ell}=\varepsilon_{\text{total}}/24$). However, we argue this is suboptimal. As \cref{fig:bitplane} and \cref{fig:five_in_row} clearly show, not all bits are created equal. The Most Significant Bits (MSBs) of the Luma (Y) channel contain the vast majority of the image's structural information. Wasting privacy budget on noisy, low-impact Least Significant Bits (LSBs) and less important chroma channels degrades utility for no meaningful privacy gain.

To formalize this intuition, we therefore frame budget allocation as a constrained optimization problem. Let $\varepsilon_{c,b}$ be the privacy budget allocated to bit-plane $b\in\{1,...,8\}$ of color channel $c\in\{\text{Y,Cb,Cr}\}$. Our objective is to find the set of budgets $\{\varepsilon_{c,b}\}$ that minimizes weighted distortion (total utility loss), subject to a fixed total budget:
\begin{equation}
\begin{aligned}
\min_{\{\varepsilon_{c,b}\}} \quad & 
\sum_{c,b} \frac{W_{c,b}}{\varepsilon_{c,b}} \\[4pt]
\text{s.t.} \quad & 
\sum_{c,b} \varepsilon_{c,b} = \varepsilon_{\mathrm{total}}, 
& \varepsilon_{c,b} \ge 0.
\end{aligned}
\end{equation}
The effectiveness of this optimization hinges on defining principled importance weights, $W_{c,b}$. We derive these from two well-established properties of digital images and human perception:
\begin{itemize}
\item \textbf{Color channel sensitivity ($w_c$):} As modern CNNs are robust to reduced chroma information~\cite{gueguen2018faster} and inspired by chroma subsampling (\eg, JPEG compression \cite{125072}). We assign a higher weight to the Y color channel. Specifically, we set $w_{\mathrm{Y}} = 4$, $w_{\mathrm{Cb}} = w_{\mathrm{Cr}} = 1$.
\item \textbf{Bit-plane significance ($w_b$):} Each bit contributes $2^{b-1}$ to the pixel value. Thus, $w_b = 2^{b-1}$ reflects its impact on image fidelity.
\end{itemize}
The total weight is the product, $W_{c,b} = w_c \cdot w_b$. With the distortion of randomized response being inversely related to the budget, we solve this optimization using Lagrange multipliers (details in the supplementary):
\begin{equation}
\boxed{
\varepsilon_{c,b} = 
\varepsilon_{\text{total}} \cdot 
\frac{
    \sqrt{W_{c,b}}
}{
    \sum_{i \in \{\mathrm{Y}, \mathrm{Cb}, \mathrm{Cr}\}}
    \sum_{j=1}^{8} \sqrt{W_{i,j}}
}
}.
\end{equation}
This allocation technique ensures that the most critical information receives the most privacy budget (\eg, the least noise). After perturbation, each pixel is reconstructed via:
\begin{equation} 
\tilde{x} = 
\sum_{\ell=1}^{8} 
2^{8-\ell} \cdot \tilde{x}_{\ell},
\end{equation}
The resulting image provides a formal $\varepsilon$-LDP guarantee while remaining fully compatible with standard vision models such as recognition and classification. 
\subsection{Privacy analysis}\label{sec:privacyana}
We now prove that LDP-Slicing provides a rigorous per-pixel LDP guarantee (see the full proof in the supplementary).
\begin{theorem}[LDP-Slicing satisfies pixel-level $\varepsilon_{\text{total}}$–LDP]
\label{thm:dpslicing_pixel}
Let $\varepsilon_{\text{total}} = \sum_{c,b} \varepsilon_{c,b}$ be the per-pixel privacy budget. LDP-Slicing guarantees $\varepsilon_{\text{total}}$-LDP for each pixel.
\end{theorem}

\begin{proof}[Proof sketch]
The proof relies on two basic properties of LDP:
\begin{enumerate}
    \item \textbf{Per-pixel composition:}  
    Our mechanism applies randomized response independently to each of the 24 bits that constitute a color pixel.  
    By the composition~\cref{thrm:seqcom}, the full pipeline is  
    $\left(\sum_{c,b}\varepsilon_{c,b}\right)$-LDP, which equals $\varepsilon_{\text{total}}$-LDP.
    \item \textbf{Immunity to post-processing:}  
    The final reconstruction step is a data-independent function applied to the privatized bits.  
    By the post-processing~\cref{thrm:postpo}, such transformations cannot weaken privacy.
\end{enumerate}
In particular, the wavelet pruning is a public, deterministic pre-processing step and does not consume the privacy budget $\varepsilon$. Therefore, the entire LDP-Slicing pipeline guarantees that each output pixel is $\varepsilon_\text{total}$-LDP with respect to the original pixel value.
\end{proof}

\noindent\textbf{The meaning of privacy guarantee.}
A formal proof is necessary, but what does this guarantee mean for the adversary defined in \cref{sec:attack} ? The $\varepsilon$-LDP guarantee provides a hard upper bound on the adversary's ability to distinguish between any two privatized images. This is captured by the Total Variation (TV) distance, which measures the maximum possible statistical difference between two distributions. A small TV distance means the outputs are statistically similar. For LDP-Slicing, this distance is tightly bounded:
 \begin{corollary}[Total Variation bound under $\varepsilon$‑LDP]
\label{lem:tv-bound}
Let $\mathcal{M}$ be an $\varepsilon$‑LDP mechanism (\cref{def:dp}). For any two inputs
$x,x'\!\in\!\mathcal{X}$, and
$P=\mathcal{L}(\mathcal{M}(x))$ and
$Q=\mathcal{L}(\mathcal{M}(x'))$.  Then
\begin{align}
  {\mathrm{TV}}(P,Q)
  &= \sup_{S\subseteq\mathcal{X}}\lvert P(S)-Q(S)\rvert
   \nonumber\\[2pt]
  &\le
  \frac{e^{\varepsilon}-1}{e^{\varepsilon}+1}
   = \tanh\!\bigl(\varepsilon/2\bigr).
  \label{eq:tv-bound}
\end{align}
Consequently, in the identity distinguishing attack in \cref{def:identity-disthin}, the adversary's advantage is upper bounded by:
\begin{align}
\label{eq:adv_advantage}
  \Adv^{\mathrm{link}}_{\mathcal{M}}
  \;\le\;
  \tfrac12\,\tanh\!\bigl(\varepsilon/2\bigr).
\end{align}
The full proof is in the supplementary.
\end{corollary}
Another crucial implication of our work is how pixel-level LDP provides effective image-level protection. Although composing the privacy bounds across all pixels implies a loose worst-case theoretical guarantee for the full image, semantic identity inherently relies on spatial correlation (\eg, contours and textures). By applying independent binary randomized response to every pixel's bit-planes, LDP-Slicing guarantees strict $\varepsilon$-LDP indistinguishability for every individual bit. This independent perturbation severely disrupts structural dependencies and erodes global semantic information, as evidenced by the low attack advantages in Sec.~\ref{sec:attack}.


\begin{table}[t]
\caption{\textbf{Benchmarks on face recognition accuracy (\%).} LDP-Slicing is compared with a non-private baseline, heuristic methods, and methods with formal privacy guarantees. \textbf{Bold} indicates the best result among DP/LDP methods.}
\label{tab:perf}
\centering
\resizebox{\columnwidth}{!}{
\begin{tabular}{l l c c c c}
\toprule
\textbf{Method} & \textbf{DP/LDP Guarantee} & \textbf{AgeDB-30} & \textbf{LFW} & \textbf{CPLFW} & \textbf{CALFW} \\
\midrule
ArcFace      & None         & 97.88 & 99.77 & 92.77 & 96.05 \\
\midrule
InstaHide  & None         & 79.58 & 96.53 & 81.03 & 86.24 \\
Cloak           & None         & 92.60 & 98.91 & 83.43 & 92.18 \\
PPFR-FD      & None         & 97.23 & 99.69 & 90.19 & 95.60 \\
DuetFace    & None         & 97.93 & 99.82 & 92.77 & 96.10 \\
PartialFace  & None      & 97.79 & 99.80 & 92.03 & 96.07 \\
ProFace       & None         & 92.81 & 98.27 & 88.17 & 93.20 \\
AdvFace      & None         & 92.57 & 98.45 & 83.73 & 93.62 \\
MinusFace   & None         & 97.57 & 99.78 & 91.90 & 95.90 \\
\midrule
PEEP        & Feature-level & 87.47 & 98.41 & 79.58 & 90.06 \\
DCTDP          & Block-level & 94.37 & 99.48 & 90.60 & 93.47 \\
\textbf{LDP-Slicing (Ours)} & \textbf{Pixel-level} & \textbf{96.68} & \textbf{99.75} & \textbf{91.08} & \textbf{96.02} \\
\bottomrule
\end{tabular}}
\end{table}

\section{Experiments}
We evaluate LDP-Slicing on two challenging domains: Privacy-Preserving Face Recognition (PPFR) and Privacy-Preserving image classification. 

\noindent\textbf{Datasets.} For face recognition,  we train on the MS1M-ArcFace (MS1MV2) dataset \cite{guo2016msceleb1mdatasetbenchmarklargescale} and evaluate on four standard benchmarks testing performance across pose, age, and quality variations: LFW \cite{lfw}, CPLFW \cite{cplwf}, and CALFW \cite{calwf}, and AgeDB-30\cite{agedb}. For image classification, we use the CIFAR-10 and CIFAR-100 \cite{cifar100} datasets for both training and testing. For reconstruction attack, we train inversion models on public BUPT dataset \cite{zhang2020class}.

\noindent\textbf{Implementation details.} For PPFR, we use a pretrained ResNet-50 backbone \cite{7780459} with ArcFace loss \cite{arcface}, trained for 24 epochs with SGD \cite{robbins1951stochastic}. For small privacy budgets (\eg, $\varepsilon = 2.4$), we use gradient clipping (max norm = 1) to mitigate gradient vanishing. Unless specified, we set default testing privacy budget $\varepsilon_\text{total} = 20$. For classification, we train a ResNet-56 from scratch. Full hyperparameter details are provided in the supplementary.
\begin{figure}[t]
  \centering
  \includegraphics[width=\columnwidth]{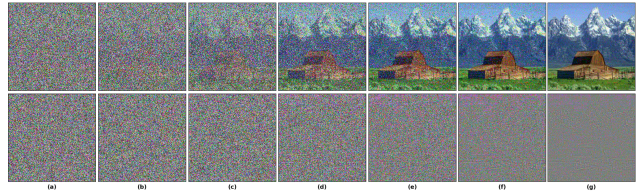}
  \caption{\textbf{Visual comparison between LL pruning (Bottom) versus without pruning (Top).}
  \textbf{(a)} $\varepsilon_\text{total}=1$. 
  \textbf{(b)} $\varepsilon_\text{total}=2.4$. 
  \textbf{(c)} $\varepsilon_\text{total}=5.2$. 
  \textbf{(d)} $\varepsilon_\text{total}=12$. 
  \textbf{(e)} $\varepsilon_\text{total}=20$. 
  \textbf{(f)} $\varepsilon_\text{total}=32$. 
  \textbf{(g)} $\varepsilon_\text{total}=58$.}
  \label{fig:aba_dc}
\end{figure}
\subsection{Privatized image in downstream tasks}                      
\noindent\textbf{Methods for comparison.} 
For face recognition, we benchmark LDP-Slicing against three categories of methods: (1) The non-private ArcFace baseline \cite{arcface}, (2) Heuristic methods that obfuscate data but offer no formal privacy guarantees (InstaHide \cite{huang2021}, Cloak \cite{cloak}, PPFR-FD \cite{ppfr-fd}, DuetFace \cite{duetface}, and MinusFace \cite{minusface}), (3) Methods with formal DP guarantees including the feature-level PEEP \cite{peep} (with $\varepsilon =5$ per PCA components) and block-level DCTDP \cite{dpdct} (with $\varepsilon_\text{mean}=0.5$). For image classification, we compare LDP-Slicing against a non-private ResNet-56 as baseline and centralized DP method DP-SGD \cite{Abadi_2016}. We set $\delta$ to $1e-10$ to force it close to pure-DP. 

\noindent\textbf{Privacy-preserving face recognition.}
LDP-Slicing in \cref{tab:perf} consistently outperforms all DP/LDP-based methods on all face recognition benchmarks and close to many heuristic SOTAs. On LFW and CALFW, our performance is nearly identical to the non-private baseline. DCTDP reports privacy at the block-level, which is not directly comparable to our per-pixel notion. For a fair comparison, we provide a theoretical upper-bound conversion of their budget in supplementary, which suggests our privacy guarantee is 5 times stricter than DCTDP.
\noindent\textbf{Privacy-preserving image classification.}
For image classification benchmarks, LDP-Slicing decisively outperforms DP-SGD at all privacy budgets $\varepsilon \leq12$ on CIFAR-10 and all budgets on CIFAR-100.
    
\begin{figure}[t]                     
  \centering
 \includegraphics[width=0.8\columnwidth]{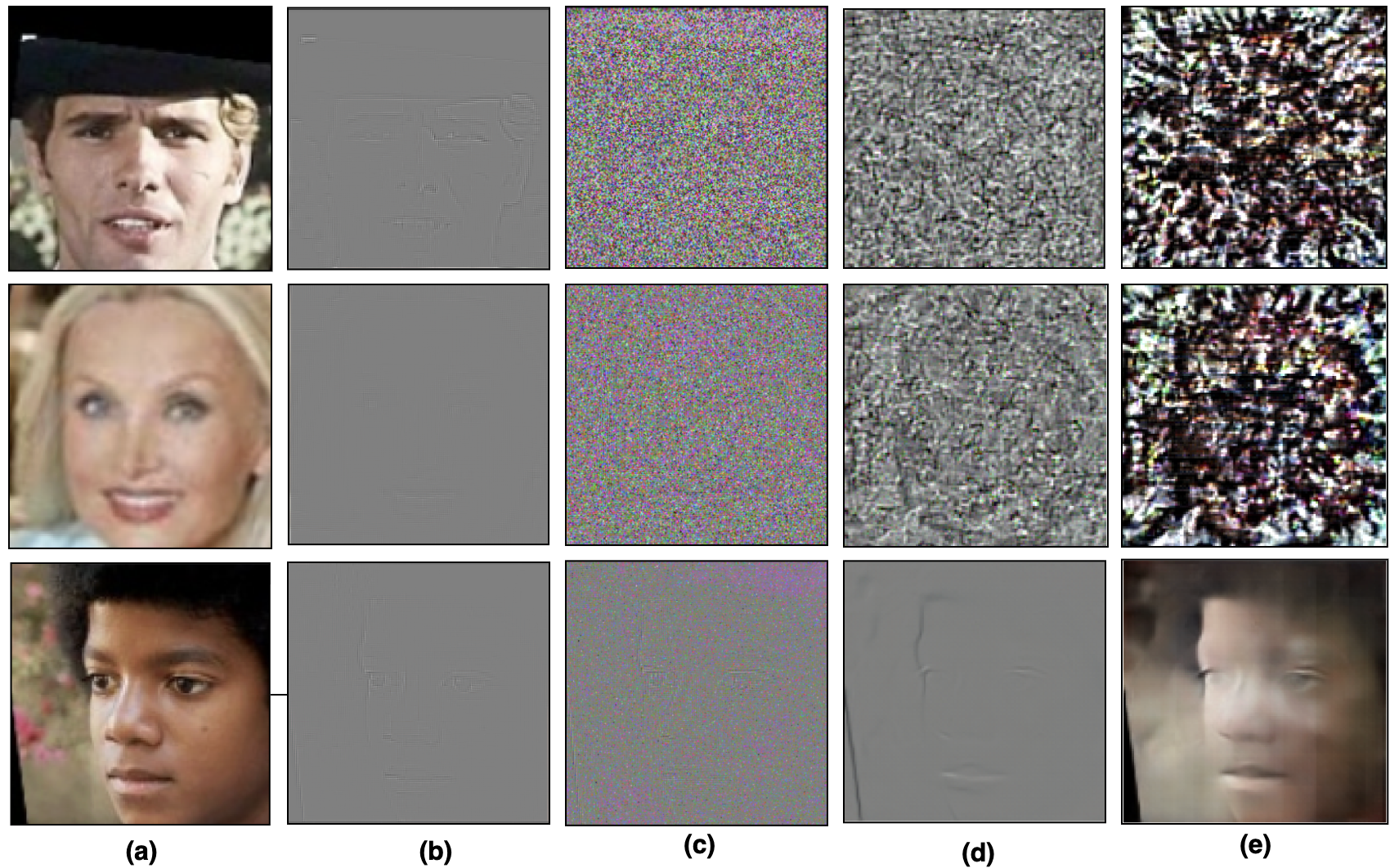}
  \caption{\textbf{Resilience to a white-box reconstruction attack.} The adversary has full knowledge of our pipeline and trains a specialized two-stage inversion model. From top to bottom $\varepsilon_{total}=5.2$, $\varepsilon_{total}=20$, $\varepsilon_{total}=58$ \textbf{(a)} The original images. \textbf{(b)} original images after LL removed. \textbf{(c)} LDP-Slicing privatized images. \textbf{(d)} LL recovered images. \textbf{(e)} Final recovered images.}
\label{fig:white_attack}
\end{figure}
\begin{figure}[t]                     
  \centering
  \includegraphics[width=0.8\columnwidth]{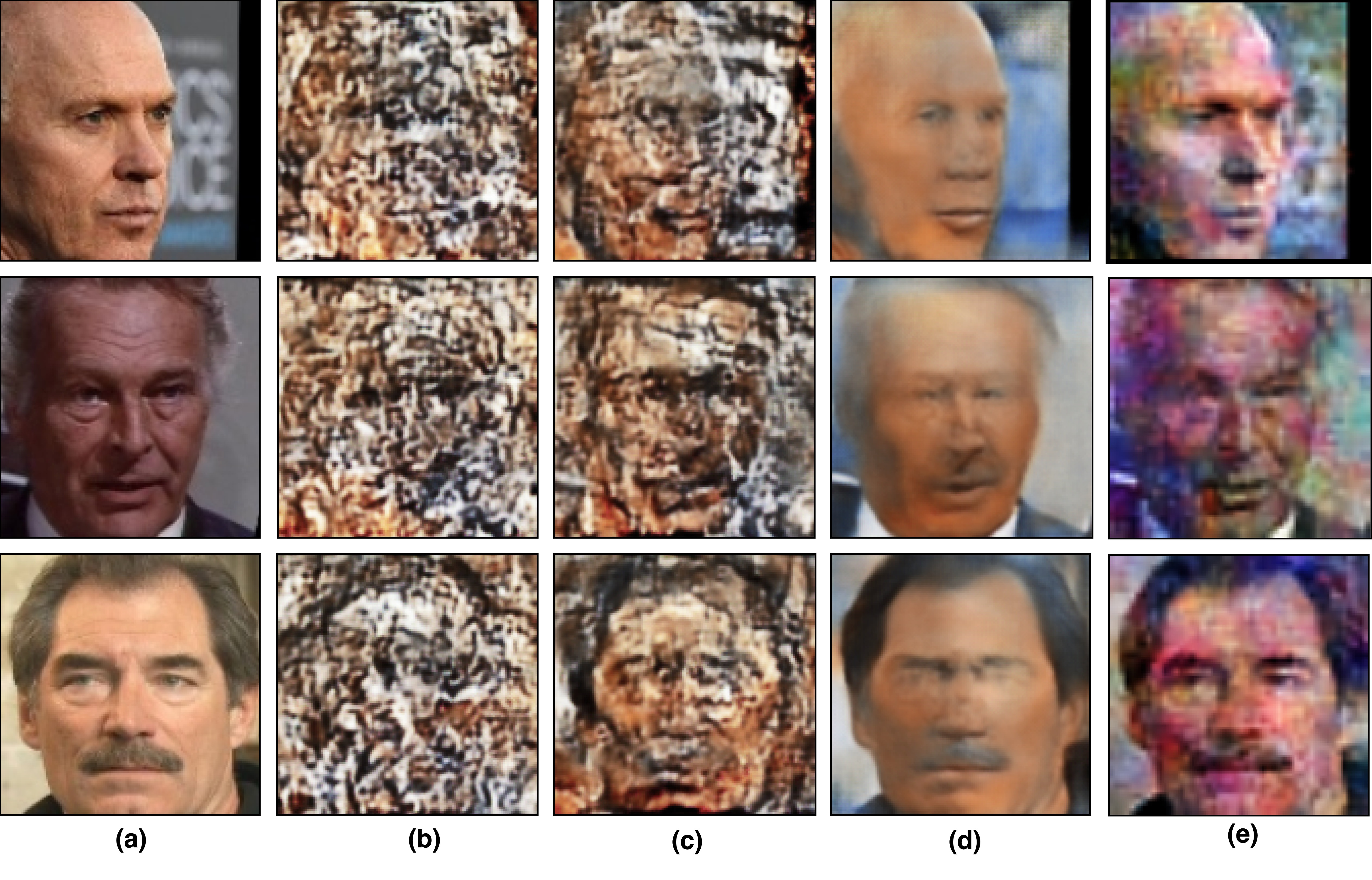}
  \caption{\textbf{Resilience to black-box reconstruction attack.} \textbf{(a)} The original images. Recovered LDP-Slicing image from \textbf{(b)} $\varepsilon_{total}=5.2$, \textbf{(c)} $\varepsilon_{total}=20$, and  \textbf{(d)} $\varepsilon_{total}=58$. \textbf{(e)} Recovered \textbf{DCTDP} image ($\varepsilon_{mean}=0.5$).}
  \label{fig:blk_attack}
\end{figure}
\begin{figure*}[t] 
  \centering
  \begin{subfigure}{0.45\textwidth}
    \centering
    \includegraphics[width=0.55\linewidth]{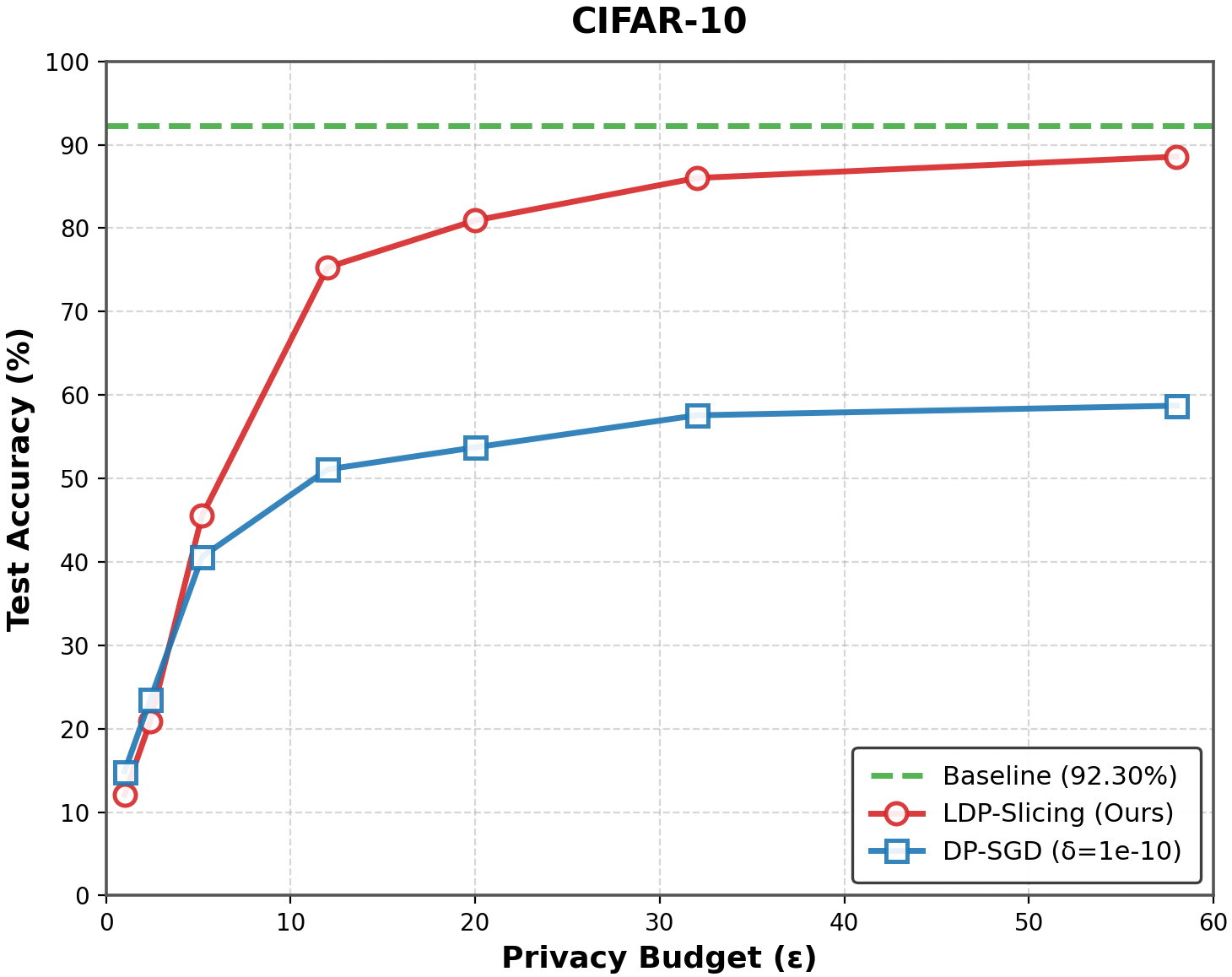}
    \caption{\textbf{Privacy-utility trade-off on CIFAR-10.}}
    \label{fig:cifar10_cls}
  \end{subfigure}
  \hfill
  \begin{subfigure}{0.45\textwidth}
    \centering
    \includegraphics[width=0.55\linewidth]{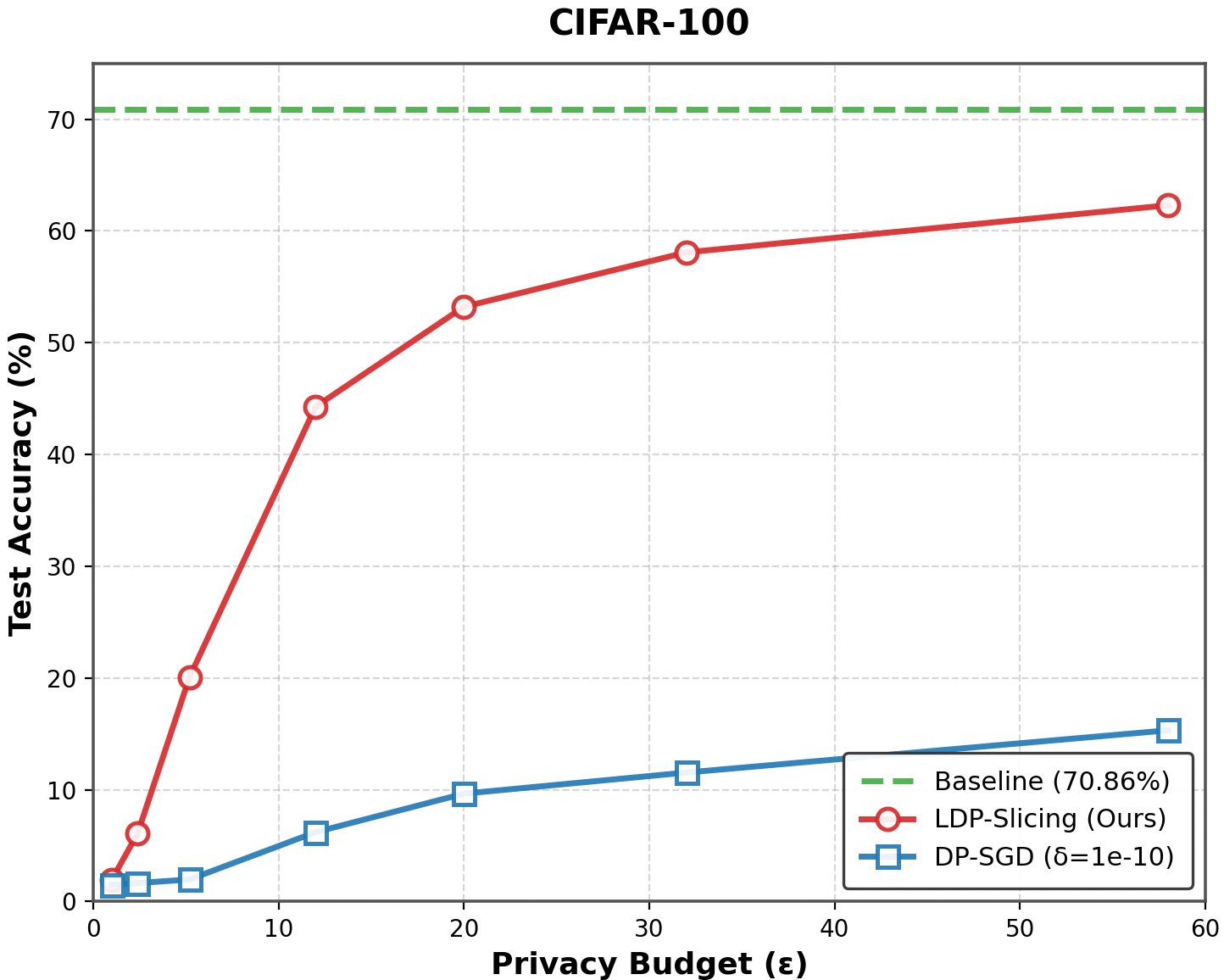}
    \caption{\textbf{Privacy-utility trade-off on CIFAR-100.}}
    \label{fig:cifar100_cls}
  \end{subfigure}
  \caption{\textbf{Privacy-utility trade-off on privacy-preserving image classification.} LDP-Slicing is compared against the strong centralized DP-SGD model \cite{Abadi_2016}. On both CIFAR-10 \textbf{(a)} and CIFAR-100 \textbf{(b)}, our local model demonstrates a clear and growing advantage as the budget moves into a more practical range, comparing with the centralized trusted-curator model.}
\end{figure*}
\subsection{Robustness against privacy attack}\label{sec:attack}
\noindent\textbf{White-box reconstruction attack.} We evaluate our method against a worst-case adversary who has full knowledge of our algorithm and trains a custom two-stage inversion model (denoiser + LL recovery network) to reverse our protection. First, a denoiser attempts to invert the LDP mechanism. It uses an autoencoder architecture \cite{bank2021autoencoders} with eight parallel attention modules \cite{vaswani2023attentionneed}, one for each bit-plane, to predict the original bits from the noisy input. Subsequently, a LL Recovery Network attempts to predict the pruned low-frequency LL information. We use a full-scale U-Net \cite{ronneberger2015u} to predict the LL band. As shown in \cref{fig:white_attack}, under medium to low privacy budget, the adversary fails to reconstruct a recognizable face. 
\begin{figure}[t]
  \centering
  \includegraphics[width=0.8\columnwidth]{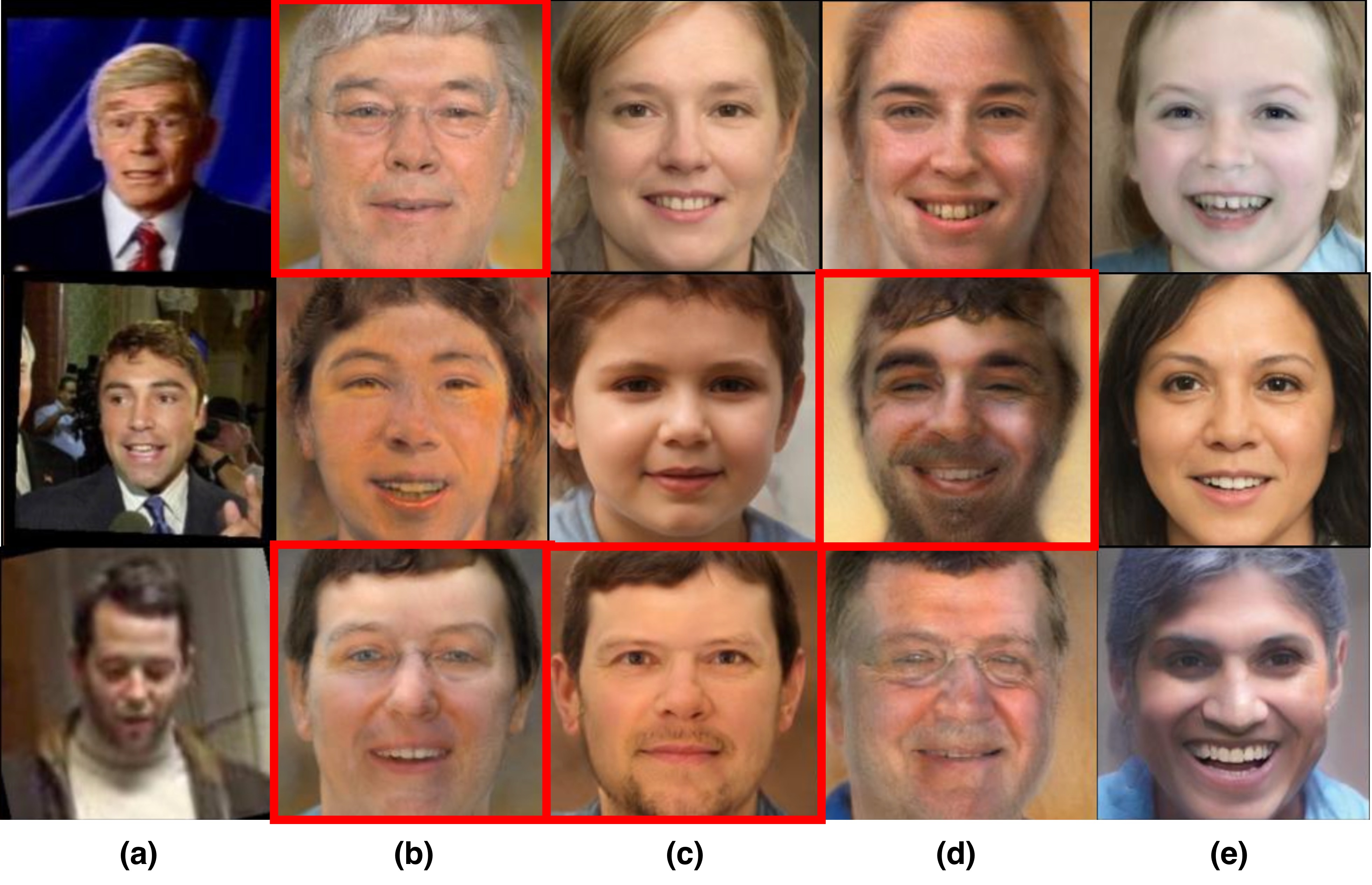}
  \caption{\textbf{Visual comparison on $\text{Map}^2$V \cite{zhang2024validating} reconstruction attack on PPFR SOTAs.} Red squares highlight reconstructions that leak facial features. \textbf{(a)} Target images. \textbf{(b)} Reconstruction from DCTDP \cite{dpdct}. \textbf{(c)} Reconstruction from DuetFace \cite{duetface}. \textbf{(d)} Reconstruction from PartialFace \cite{partialface}. \textbf{(e)} Reconstruction from LDP-Slicing (ours).}
  \label{fig:map2vattack}
\end{figure}

\noindent\textbf{Black-box reconstruction attack.} In the more realistic black-box setting, we assume an adversary who has no knowledge of our algorithm but can query it to create a original/privatized image pairs. The adversary then trains a powerful U-Net to reconstruct original images. As shown in \cref{fig:blk_attack}, the attack model reconstructs nearly identifiable faces from images protected by the competing DCTDP method. In contrast, even under a significantly relaxed privacy budget of $\varepsilon_{total}=58$, our LDP-Slicing images remain highly distorted. To simulate a truly SOTA adversary, we further evaluate against the powerful generative prior of a StyleGAN-based attack, $\text{Map}^2$V \cite{zhang2024validating}. As shown in \cref{fig:map2vattack}, while this SOTA attack successfully recovering some attributes from competing methods, it fails to break LDP-Slicing.

\noindent\textbf{Identity distinguishing attack.} Beyond visual reconstruction, we simulate an identity-distinguishing attack. Here, the adversary has a public unprotected photo of a person $I_B$, and wants to link a privatized image $\tilde{I_A}$ to that same person. This models a real-world attempt to track and link a user’s activity across different anonymized services. We train Wide-ResNet-28 as a classifier to determine if $\tilde{I_A}$ and $I_B$ are the same person (on BUPT dataset). LDP-Slicing consistently maintains a lower attack advantage than DCTDP across all datasets, as illustrated in \cref{tab:distin}.
\begin{table}[t]
\caption{\textbf{Identity distinguishing attack advantage (\%).} The table shows the adversary's advantage rate in \cref{def:identity-disthin}. A lower advantage indicates a stronger protection.}
\label{tab:distin}
\centering
\resizebox{\columnwidth}{!}{
\begin{tabular}{l c c c c c c}
\toprule
\textbf{Method} & \textbf{CIFAR10} & \textbf{CIFAR100} & \textbf{AgeDB-30} & \textbf{LFW} & \textbf{CALFW} & \textbf{CPLFW} \\
\midrule
PartialFace \cite{partialface} &3.55 & 4.34 & 0.43  &16.28& 9.0 & 6.18  \\ DuetFace \cite{duetface} & 5.07 & 7.14& 6.97& 19.75& 12.98 & 8.89 \\
DCTDP \cite{dpdct}   & 4.97 & 5.98 & 5.65 & 12.58 & 10.92 & 10.36\\
\textbf{LDP-Slicing (Ours)} & \textbf{0.25} & \textbf{0.1} & \textbf{0.42} & \textbf{4.5} & \textbf{3.87} & \textbf{1.62} \\
\bottomrule
\end{tabular}}
\end{table}
\subsection{Ablation study}
\begin{table}[!t]
  \centering
  \scriptsize
  \setlength{\tabcolsep}{2pt}
  \caption{
    \textbf{Ablation study of key components of LDP-Slicing}. Results show face recognition accuracy (\%) on four benchmarks, with a fixed privacy budget of $\varepsilon_{\text{total}}=20$. Each ablative variant modifies a single component of our full method (D):
    (A) without LL pruning, (B) with uniform privacy budget, and (C) DCT-based DC pruning with full utility-aware optimization.
  }
  \label{tab:abla}
  \resizebox{\columnwidth}{!}{%
    \begin{tabular}{lcccc}
      \toprule
      \textbf{Configuration} & \textbf{AgeDB-30} &
      \textbf{LFW} & \textbf{CPLFW} & \textbf{CALFW} \\
      \midrule
      ArcFace (baseline)                        & 97.88 & 99.77 & 92.77 & 96.05 \\
      (A) w/o LL Pruning                        & 96.90 & 99.77 & 91.67 & 96.05 \\
      (B) LL Pruning + Uniform Budget                      & 89.82 & 99.35 & 86.73 & 94.13 \\
      (C) DC Pruning + Utility-aware optimization & 95.58 & 99.53 & 89.53 & 95.77\\
      (D) LL Pruning + Utility-aware optimization (LDP-Slicing)     & 96.68 & 99.75 & 91.08 & 96.02 \\
      \bottomrule
    \end{tabular}
  }
\end{table}
\noindent\textbf{Effect of LL pruning.}
Comparing our full method (D) with a variant without LL pruning (A) shows the utility trade-off made for perceptual privacy. While LL pruning yields a minor accuracy gain, we show in \cref{fig:aba_dc} that removing it leaves images visually identifiable to human observers at large  $\varepsilon_{total}$. We also show in the supplementary that under moderate settings, LDP-Slicing remains robust even without LL pruning. Furthermore, comparing our DWT-based pruning (D) to a DCT-based low-frequency removal alternative (C) confirms that DWT preserves better utility.

\noindent\textbf{Effect of utility-aware budget optimization.} We compare our full, optimized budget allocation (C) against a naive uniform allocation (B). The uniform approach results in a catastrophic performance degradation across all benchmarks (e.g., a 6.86\% absolute drop on AgeDB-30). This proves our hypothesis of utility-aware optimizing the privacy budget based on the significance of each bit plane and color channel.

\noindent\textbf{Impact of the total privacy budget $\varepsilon_{total}$.}
Finally, we analyze the privacy-utility trade-off across a range of budgets. As shown in \cref{fig:aba_dc} and \cref{tab:var_bugdet}, our method exhibits a predictable trade-off: as the privacy guarantee becomes stronger (lower $\varepsilon$), utility declines. 
\begin{table}[!t]
  \centering
  \scriptsize
  \setlength{\tabcolsep}{2pt}
  \caption{\textbf{Privacy-utility trade-off on face recognition accuracy (\%)}. As the privacy budget $\varepsilon_{\text{total}}$ decreases, utility across all benchmarks declines smoothly, aligning with theoretical expectations. PSNR values also correlate with the applied privacy budget.}
  \label{tab:privacy_utility}
  \resizebox{\columnwidth}{!}{%
    \begin{tabular}{lcccccc}
      \toprule
      \textbf{Privacy Budget} & \textbf{AgeDB-30} &
      \textbf{LFW} & \textbf{CPLFW} & \textbf{CALFW} &
      $\boldsymbol{\mathrm{PSNR}_{\text{avg}}}$ \\ \midrule
      $\varepsilon_{\text{total}}=1$   & 50.95 & 52.88 & 51.45 & 52.31 & 6.90\,dB \\
      $\varepsilon_{\text{total}}=2.4$ & 54.12 & 53.62 & 51.85 & 53.83 & 6.91\,dB \\
      $\varepsilon_{\text{total}}=5.2$ & 68.89 & 89.31 & 71.30 & 79.15 & 6.94\,dB \\
      $\varepsilon_{\text{total}}=12$  & 93.70 & 99.30 & 88.70 & 95.33 & 7.30\,dB \\
      $\varepsilon_{\text{total}}=20$   & 96.68 & 99.75 & 91.08 & 96.02 & 7.62\,dB \\
      $\varepsilon_{\text{total}}=32$  & 96.46 & 99.77 & 91.93 & 96.04 & 8.51\,dB \\
      $\varepsilon_{\text{total}}=58$  & 97.45  & 99.77 & 92.07 & 96.05 & 9.94\,dB \\
      \bottomrule
      \label{tab:var_bugdet}
    \end{tabular}
  }
\end{table}
\begin{table}[t]
\caption{\textbf{Storage and transmission overhead analysis.} Overhead is reported as a multiple of the size of a standard image (our method, $\times 1$). Lower is better. LDP-Slicing introduces zero overhead.}
\label{tab:storage}
\centering
\small
\setlength{\tabcolsep}{6pt}
\renewcommand{\arraystretch}{1.2}
\resizebox{\columnwidth}{!}{
\begin{tabular}{l | c c c c c c c}
\toprule
\textbf{Method}  
& Ours 
& \cite{minusface} 
& \cite{advface} 
& \cite{partialface} 
& \cite{ppfr-fd} 
& \cite{duetface} 
& \cite{dpdct} \\
\midrule
\textbf{Storage} $\downarrow$
& $\times 1$ 
& $\times 1$
& $\times 5.3$ 
& $\times 9$
& $\times 36$
& $\times 54$
& $\times 63$ \\
\bottomrule
\end{tabular}
}
\end{table}
\subsection{Complexity and compatibility} 
\noindent\textbf{Time complexity.} On a consumer-grade Apple M4 chip, we test LDP-Slicing over one hundred $112\times112$ images and the average processing time is 5.5 ms (throughput of 232 images/s), comparing existing SOTA MinusFace \cite{minusface}, which takes 68 ms. 
The complexity scales linearly with the number of pixels $N$, with the overall time complexity $\Theta(N)$. This low computational overhead makes LDP-Slicing a practical solution for edge devices.  \\
\noindent\textbf{Storage and transmission overhead.} As shown in \cref{tab:storage}, LDP-Slicing incurs zero storage or transmission overhead, outputting a standard image. This is a vast improvement over prior works like DCTDP \cite{dpdct} and DuetFace \cite{duetface}, which require representations that are $\times 63$ and $\times 54$ larger.\\ 
\noindent\textbf{Compatibility and Downstream Invariance.} Because LDP-Slicing outputs a standard image, it is inherently invariant to downstream tasks and can directly replace unprotected images without architectural modifications. We validate this task-agnostic robustness by evaluating our MS1MV2-trained model zero-shot on cross-dataset benchmarks (VGGFace2 \cite{vggface2}, CelebA \cite{celebA}), and on a completely different domain with higher-resolution medical data (Chest X-Ray \cite{medical}). Across these diverse settings, LDP-Slicing consistently preserves high utility. Furthermore, by the post-processing property of LDP, any model trained on these images inherits the rigorous privacy guarantee. Full downstream quantitative results are detailed in the supplementary. 
\section{Conclusion}
This paper addressed the challenge of enabling source-level privacy for images without sacrificing utility. Our framework, LDP-Slicing, transforms pixels into a bit-plane representation where LDP is natively effective. Combined with perceptual obfuscation and utility-aware budget optimization, it guarantees rigorous, source-level privacy while preserving downstream task accuracy. By making LDP practical for standard vision pipelines, this work unlocks zero-trust deployment in sensitive domains like medical imaging, paving the way for future extensions into video streams.

\clearpage
\setcounter{page}{1}
\maketitlesupplementary
\setcounter{section}{0}
\renewcommand\thesection{\Alph{section}}
\renewcommand\thesubsection{\thesection.\arabic{subsection}}
\begin{figure*}[!t]
  \centering
    \includegraphics[width=\textwidth]{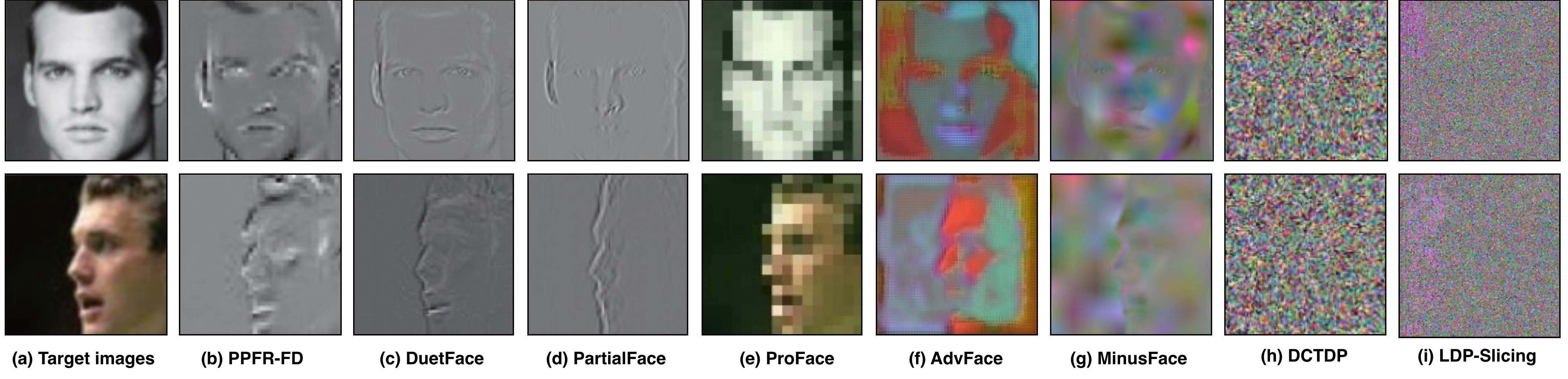}%
  \caption{\textbf{Qualitative comparison against other SOTAs.}  We compare LDP-Slicing (i) against (b) PPFR-FD \cite{ppfr-fd}, (c) DuetFace \cite{duetface}, (d) PartialFace \cite{partialface}, (e) ProFace \cite{proface}, (f) AdvFace \cite{advface}, (g) MinusFace \cite{minusface}, and DCTDP \cite{dpdct}. DCTDP and LDP-Slicing achieve a higher level of perceptual obfuscation.}
  \label{fig:Qualitativecmp}
\end{figure*}

This supplementary material provides additional details about the proposed LDP-Slicing method. Specifically, we provide:
\begin{itemize}
  \item Formal definitions of the reconstruction threat model (\cref{sec:thret}).
  \item Derivation of the utility-aware budget optimization. 
  \item Complete proofs for the pixel-level Local Differential Privacy (LDP) guarantee and Total Variation bound. 
  \item Training configurations and hyperparameters for all benchmarks.
  \item Additional Experiments. 
  \item Ethical discussion.
\end{itemize}

\section{Formal attack definition} \label{sec:attackdef}
We formalize the adversarial attack models discussed in our threat model in \cref{sec:thret}. 
\begin{definition}[Reconstruction attack~\cite{carlini2021}]
Let $S=\{(x_1,y_1),\dots,(x_n,y_n)\}$ be the sensitive dataset, and let $\tilde{S}=(\tilde{X},\tilde{Y}) \leftarrow \mathcal{M}(S)$ be the privatized dataset revealed to an adversary $\mathcal{A}$.
Given $\tilde{S}$, $\mathcal{A}$ outputs a candidate record $x^* \in \mathcal{X}$.
We say that $\mathcal{A}$ succeeds in a reconstruction attack if
\[
  \exists\, i \in [n] \text{ such that } d(x^*, x_i) \le r,
\]
where $d:\mathcal{X}\times\mathcal{X}\to\mathbb{R}_{\ge0}$ is a context-dependent metric
and $r>0$.
\end{definition}
In the context of face images, a successful reconstruction means that $x^*$ preserves identifiable features of the original person (\eg, eyes and nose). An ideal privacy mechanism should make it computationally infeasible for any adversary to produce such an $x^*$ from $\tilde{x}$ with high confidence. 

\section{Derivation of utility-aware budget optimization}
In \cref{sec:optmization} of the main paper, we formulate the following
constrained optimization problem:
\[
\begin{aligned}
& \underset{\{\varepsilon_{c,b}\}}{\text{minimize}}
& &  \sum_{c,b} \frac{W_{c,b}}{\varepsilon_{c,b}} \\
& \text{subject to}
& & \sum_{c,b} \varepsilon_{c,b} = \varepsilon_{\text{total}}, \qquad
    \varepsilon_{c,b} > 0,
\end{aligned}
\]
where $W_{c,b} > 0$ represents the importance weight for bit-plane $b$ of
channel $c$, and $\varepsilon_{c,b}$ is the allocated privacy budget.

We solve this using the method of Lagrange multipliers:
\[
\mathcal{L}(\{\varepsilon_{c,b}\}, \lambda)
=
\sum_{c,b} \frac{W_{c,b}}{\varepsilon_{c,b}}
+ \lambda \left( \sum_{c,b} \varepsilon_{c,b} - \varepsilon_{\text{total}} \right),
\]
where $\lambda \in \mathbb{R}$ is the Lagrange multiplier. Taking partial derivatives with respect to each decision variable
$\varepsilon_{c,b}$ and setting them to zero yields:
\[
\frac{\partial \mathcal{L}}{\partial \varepsilon_{c,b}}
= -\frac{W_{c,b}}{\varepsilon_{c,b}^2} + \lambda = 0.
\]
Solving for $\varepsilon_{c,b}$, we obtain:
\begin{equation}
\varepsilon_{c,b}
= \frac{\sqrt{W_{c,b}}}{\sqrt{\lambda}}.
\label{eq:eps_lambda}
\end{equation}
Since $W_{c,b} > 0$ and $\varepsilon_{c,b} > 0$, this implies $\lambda > 0$.
Substituting \eqref{eq:eps_lambda} into the budget constraint
$\sum_{c,b} \varepsilon_{c,b} = \varepsilon_{\text{total}}$, we obtain:
\[
\sum_{c,b} \varepsilon_{c,b}
=
\sum_{c,b} \frac{\sqrt{W_{c,b}}}{\sqrt{\lambda}}
= \varepsilon_{\text{total}},
\]
Solving for $1/\sqrt{\lambda}$:
\begin{equation}
\frac{1}{\sqrt{\lambda}}
= \frac{\varepsilon_{\text{total}}}{\sum_{c',b'} \sqrt{W_{c',b'}}}.
\label{eq:inv_sqrt_lambda}
\end{equation}
Finally, substituting \eqref{eq:inv_sqrt_lambda} back into
\eqref{eq:eps_lambda} gives the final solution
\[
\varepsilon_{c,b}
= \sqrt{W_{c,b}} \cdot
  \frac{\varepsilon_{\text{total}}}{\sum_{c',b'} \sqrt{W_{c',b'}}},
\]
or equivalently,
\begin{equation}
\boxed{
\varepsilon_{c,b} = 
\varepsilon_{\text{total}} \cdot 
\frac{
    \sqrt{W_{c,b}}
}{
    \sum_{i \in \{\mathrm{Y}, \mathrm{Cb}, \mathrm{Cr}\}}
    \sum_{j=1}^{8} \sqrt{W_{i,j}}
}
}.
\end{equation}
\section{Complete proofs}
\subsection{Proof of pixel-level LDP guarantee}
\label{sec:proof_pixel_ldp}
In Theorem~\ref{thm:dpslicing_pixel}, we claim that LDP-Slicing guarantees that each individual pixel in the output image satisfies $\varepsilon_{\text{total}}$-LDP with respect to the original pixel value.
Let $\mathcal{X}$ denote the domain of possible pixel values after LL pruning, and let
\[
  \mathcal{M} : \mathcal{X} \to \mathcal{Y}
\]
be the per-pixel LDP-Slicing mechanism.
By Definition~\ref{def:dp}, this means that for any two possible pixel values $x, \tilde{x}$ and any measurable subset of possible randomized outputs $S$:
\begin{equation}
  \Pr[\mathcal{M}(x) \in S]
  \;\leq\;
  e^{\varepsilon_{\text{total}}} \cdot
  \Pr[\mathcal{M}(\tilde{x}) \in S],
  \label{eq:pixel-ldp-ineq}
\end{equation}
where $\varepsilon_{\text{total}}$ is the total privacy budget allocated per pixel. In this section, we provide the complete formal proof.
\subsection*{Setup}

Each pixel after LL pruning has Y, Cb, and Cr channels, each with 8 bits. We index the resulting $d=24$ bits as:
\[
  x \longmapsto (x_{1},\dots,x_{d}) \in \{0,1\}^d.
\]
Let $\varepsilon_\ell$ denote the privacy budget allocated to the $\ell$-th bit, where the total budget is $\varepsilon_{\text{total}} = \sum_{\ell=1}^{d} \varepsilon_\ell$.
\begin{proof}
\noindent\textbf{1. Single bit privacy.} \\
For any bit \(x_\ell\) in bit-plane \(B_\ell\), we apply the binary randomized response mechanism
  \(\RR\) with budget \(\varepsilon_\ell\).
  Let $p$ denote the probability of true response: 
  \[
    p
    \;=\;
    \frac{e^{\varepsilon_\ell}}{e^{\varepsilon_\ell}+1}
  \]
  By Definition~\ref{def:rr}, for any output \(\tilde{x}_\ell \in \{0,1\}\), we have:
  \begin{equation}
    \Prv{\RR(x_\ell) = \tilde{x}_\ell}
    =
    \begin{cases}
      p, & \text{if } x_\ell = \tilde{x}_\ell, \\[2pt] 
      1-p, & \text{if } x_\ell \neq \tilde{x}_\ell.
    \end{cases}
  \end{equation}
The worst-case likelihood ratio between two different inputs is then:
  \[
    \frac{\Prv{\RR(x_\ell) = \tilde{x}_\ell}}
         {\Prv{\RR(x_\ell) \neq \tilde{x}_\ell}}
    \;\le\;
    \frac{p}{1-p}
    \;=\;
    \frac{\displaystyle\frac{e^{\varepsilon_\ell}}{e^{\varepsilon_\ell}+1}}
         {\displaystyle\frac{1}{e^{\varepsilon_\ell}+1}}
    \;=\;
    e^{\varepsilon_\ell}.
  \]
  Thus, for every $\ell$ and for all $x_\ell,\tilde{x}_\ell \in \{0,1\}$ and all $S \subseteq \{0,1\}$:
 \begin{equation}\label{eq:perbit}
    \Prv{\RR(x_\ell) \in S}
    \;\le\;
    e^{\varepsilon_\ell} \Prv{\RR(\tilde{x}_\ell) \in S}
\end{equation}

\medskip
\noindent\textbf{2. Basic composition.} \\
For a fixed pixel $x \in \mathcal{X}$, let its $d$-bit representation after bit-plane
slicing be $(x_1,\dots,x_d) \in \{0,1\}^d$. For each bit-plane index $\ell \in \{1,\dots,d\}$, we apply the mechanism $\RR$ in parallel with budget $\varepsilon_\ell$:
\begin{equation}\label{eq:bsicmechan}
  \mathcal{M}_{1:d}(x)
  :=
  \bigl(\mathcal{M}_1(x_1),\dots,\mathcal{M}_d(x_d)\bigr)
  \in \{0,1\}^d.
\end{equation}
Where the mechanisms $\mathcal{M}_1,\dots,\mathcal{M}_d$ are run with mutually
independent on each bit. For any output pixel $\tilde{\mathbf{x}} = (\tilde{x}_1,\dots,\tilde{x}_d) \in \{0,1\}^d$ we have:
\[
  \Prv{\mathcal{M}_{1:d}(x) = \tilde{{x}}}
  \;=\;
  \prod_{\ell=1}^{d}
    \Prv{\mathcal{M}_\ell(x_\ell) = \tilde{x}_\ell}.
\]
Fix any two pixel values $x,\tilde{x} \in \mathcal{X}$ we obtain:
\begin{align*}
  \frac{\Prv{\mathcal{M}_{1:d}(x) = \tilde{{x}}}}
       {\Prv{\mathcal{M}_{1:d}({x}) \neq\tilde{{x}}}}
  &=
  \frac{\prod_{\ell=1}^{d} \Prv{\mathcal{M}_\ell(x_\ell) = \tilde{x}_\ell}}
       {\prod_{\ell=1}^{d} \Prv{\mathcal{M}_\ell({x}_\ell) \neq \tilde{x}_\ell}}
  \\[4pt]
  &=
  \prod_{\ell=1}^{d}
    \frac{\Prv{\mathcal{M}_\ell(x_\ell) = \tilde{x}_\ell}}
         {\Prv{\mathcal{M}_\ell({x}_\ell) \neq \tilde{x}_\ell}}
  \\[4pt]
  &\le
  \prod_{\ell=1}^{d} e^{\varepsilon_\ell}
  \;=\;
  e^{\sum_{\ell=1}^{d} \varepsilon_\ell}
  \;=\;
  e^{\varepsilon_{\text{total}}}.
\end{align*}
Therefore $\mathcal{M}_{1:d}$ is $\varepsilon_{\text{total}}$-LDP.

\item \textbf{3. Immunity to post-processing.}\\
We now prove formally that the reconstruction step preserves the
$\varepsilon_{\text{total}}$-LDP guarantee. Let
\[
  \mathcal{M}_{1:d} : \mathcal{X} \to \{0,1\}^d
\]
denote $\varepsilon_{\text{total}}$-LDP the LDP mechanism from \cref{eq:bsicmechan}. The reconstruction map
$f : \{0,1\}^d \to \mathcal{Y}$ is deterministic and given by:
\[
  f(\tilde{x}_1,\dots,\tilde{x}_d)
  =
  \sum_{\ell=1}^{d} 2^{d-\ell} \tilde{x}_\ell.
\]
The per-pixel LDP-Slicing mechanism is therefore:
\[
  \mathcal{M}(x) = f\bigl(\mathcal{M}_{1:d}(x)\bigr).
\]
Fix any $x,x' \in \mathcal{X}$ and any measurable subset $S \subseteq \mathcal{Y}$. Define
\[
  T = \{ z \in \{0,1\}^d : f(z) \in S \}.
\]
Then
\[
  \Prv{f(\mathcal{M}_{1:d}(x)) \in S}
  =
  \Prv{\mathcal{M}_{1:d}(x) \in T}.
\]
Since $\mathcal{M}_{1:d}$ is $\varepsilon_{\text{total}}$-LDP, we have:
\[
  \Prv{\mathcal{M}_{1:d}(x) \in T}
  \;\le\;
  e^{\varepsilon_{\text{total}}}
  \Prv{\mathcal{M}_{1:d}(\tilde{x}) \in T}.
\]
Hence LDP-Slicing is per-pixel $\varepsilon_{\text{total}}$-LDP.
\end{proof}
\subsection{Proof of total variation bound under $\varepsilon$-LDP}

\begin{proof}
Let $P$ and $Q$ be the output distributions of the LDP-Slicing mechanism $\mathcal{M}$
on pixel $x$ and $\tilde{x}$. By $\varepsilon$-LDP
(Definition~\ref{def:dp}), for all measurable $S \subseteq \mathcal{Y}$,
\[
  P(S) \le e^{\varepsilon} Q(S)
\]
The total variation distance between distribution $P$ and $Q$ is:
\[
  \mathrm{TV}(P,Q)
  = \sup_{S \subseteq \mathcal{Y}} \lvert P(S) - Q(S) \rvert.
\]

The LDP inequalities imply that the likelihood ratio between $P$ and $Q$
is bounded:
\[
  e^{-\varepsilon}
  \;\le\;
  \frac{dP}{dQ}(y)
  \;\le\;
  e^{\varepsilon}
\]
then
\[
  \mathrm{TV}(P,Q)
  \;\le\;
  \frac{e^{\varepsilon}-1}{e^{\varepsilon}+1}
  = \tanh\!\bigl(\varepsilon/2\bigr),
\]
which proves \eqref{eq:tv-bound}.

For the identity distinguishing attack in
Definition~\ref{def:identity-disthin}, let $P$ and $Q$ be the
distributions of the observation conditioned on $b=1$ and $b=0$,
respectively. Any (possibly randomized) adversary corresponds to some
decision region $S$ in which it outputs $b'=1$.
Then
\[
  \Pr[b'=b]
  = \tfrac12 \bigl(1 + P(S) - Q(S)\bigr),
\]
so:
\[
  \left|\Pr[b'=b] - \tfrac12\right|
  = \tfrac12\,\lvert P(S) - Q(S)\rvert
  \;\le\;
  \tfrac12\,\mathrm{TV}(P,Q).
\]
Combining this with the bound on $\mathrm{TV}(P,Q)$ yields:
\[
  \Adv^{\mathrm{dist}}_{\mathcal{M}}
  \;\le\;
  \tfrac12\,\tanh\!\bigl(\varepsilon/2\bigr),
\]
\end{proof}

\section{Experiments}
\subsection{Experimental details}
This subsection provides detailed information about our experimental setup and hyperparameters. We conduct all experiments on NVIDIA H100 GPUs using PyTorch.

\noindent\textbf{Image pre-processing.}
For face recognition, we resize each facial image to 112×112 pixels to align with other SOTAs. For image classification, we apply augmentation like random crop and random horizontal flip and optional cutout. 

\noindent\textbf{Privacy-preserving face recognition.} We adopt the ResNet-50 architecture with the improved residual \cite{7780459} backbone. The network is optimized using the ArcFace loss function \cite{arcface} with a feature scale $s=64$ and angular margin $m=0.4$. We use SGD \cite{robbins1951stochastic} with a momentum of 0.9 and a weight decay of $5\times10^{-4}$. Training proceeds for 24 epochs with an initial learning rate of 0.1, decayed by a factor of 10 at epochs 10, 18, and 22. We also initialize weights from a pre-trained IR-50 model and apply a 5-epoch warm-up period.

\noindent\textbf{Privacy-preserving image classification.}
For image classification on CIFAR-10 and CIFAR-100 dataset, we adopt a ResNet-56 architecture trained for 250 epochs with a batch size of 128. We use SGD with a momentum of 0.9 and a reduced weight decay of
$1\times10^{-4}$. We set the initial learning rate to 0.1 and decay by a factor of 10 at epochs 80, 160, and 200. For training with low privacy budgets (\eg, $\varepsilon=1$), we apply the gradient clipping with a max norm of 0.5. 
\subsection{Additional experiments}
We extend LDP-Slicing to the medical domain using the Chest X-Ray dataset \cite{medical} at $224\times224$ resolution. We adopt ResNet-50 as backbone and remove color weights during the budget optimization process. As shown in \cref{tab:privacy_utility_rebuttal}, LDP-Slicing retains high utility even under strict privacy budgets (\eg $\varepsilon_{total}=5.2$). Furthermore, we perform a zero-shot test of our MS1MV2 trained model directly on the VGGFace2 \cite{vggface2} and CelebA \cite{celebA} datasets. LDP-Slicing maintains a competitive utility under moderate privacy settings (in \cref{tab:privacy_utility_rebuttal}).
\subsection{Ablation on different color weights}
 We evaluate LDP-Slicing under 2 additional different weights (in \cref{tab:jpg_weight_rebuttal}). The original weight (4:1:1) consistently outperforms other weights.

\subsection{Ablation on LL pruning}
We perform an ablation study (in \cref{fig:blacbox}) by removing the perceptual obfuscation stage. Without LL pruning, LDP-Slicing can still suppress identity-defining high-frequency cues, and the inversion recovers at most blurred, non-identifying content. Although perceptual obfuscation is an auxiliary step, it allows the same $\varepsilon$ noise to be spent more on task-relevant signals. 
\begin{table}[!t]
  \centering
  \caption{\textbf{Privacy-utility trade-off (\%).} We perform additional experiment on another 3 benchmarks}
  \label{tab:privacy_utility_rebuttal}
  \resizebox{\columnwidth}{!}{%
  \begin{tabular}{l|ccccccc}
    \toprule
    \textbf{Dataset} & $\varepsilon$=1 & $\varepsilon$=2.4 & $\varepsilon$=5.2 & $\varepsilon$=12 & $\varepsilon$=20 & $\varepsilon$=32 & $\varepsilon$=58 \\ \midrule
    \textbf{Chest X-Ray} \cite{medical} & 76.60 & 88.94 & 90.54 & 91.99 & 92.15 & 92.95 & 93.43 \\
    \textbf{VGGFace2} \cite{vggface2}   & 48.95 & 50.31 & 51.71 & 66.69 & 67.15 & 68.99 & 70.64 \\
    \textbf{CelebA} \cite{celebA}       & 49.88 & 50.97 & 50.84 & 75.90 & 80.44 & 82.68 & 84.89 \\
    \bottomrule
  \end{tabular}%
  }
 
\end{table}
\begin{table}[!t]
  \centering
  \caption{\textbf{Utility under different color weight (\%).}}
  \label{tab:jpg_weight_rebuttal}
  \resizebox{\columnwidth}{!}{%
  \begin{tabular}{l c c c c c c}
    \toprule
    \textbf{Color W.} & \textbf{CIFAR10} & \textbf{CIFAR100} & \textbf{AgeDB} & \textbf{LFW} & \textbf{CALFW} & \textbf{CPLFW} \\
    \midrule
    \textbf{$4:1:1$ (Ours)} & \textbf{80.36} & \textbf{53.55} & \textbf{96.68} & \textbf{99.75} & \textbf{96.02} & \textbf{91.08} \\
    $2:1:1$ & 78.44 & 50.63 & 95.85 & 99.63 & 95.68 & 90.87 \\
    $1:1:1$ & 74.92 & 45.16 & 94.68 & 99.47 & 94.17 & 89.85 \\
    \bottomrule
  \end{tabular}%
  }
\end{table}

\begin{figure}[!t]
  \centering
  \includegraphics[width=\columnwidth]{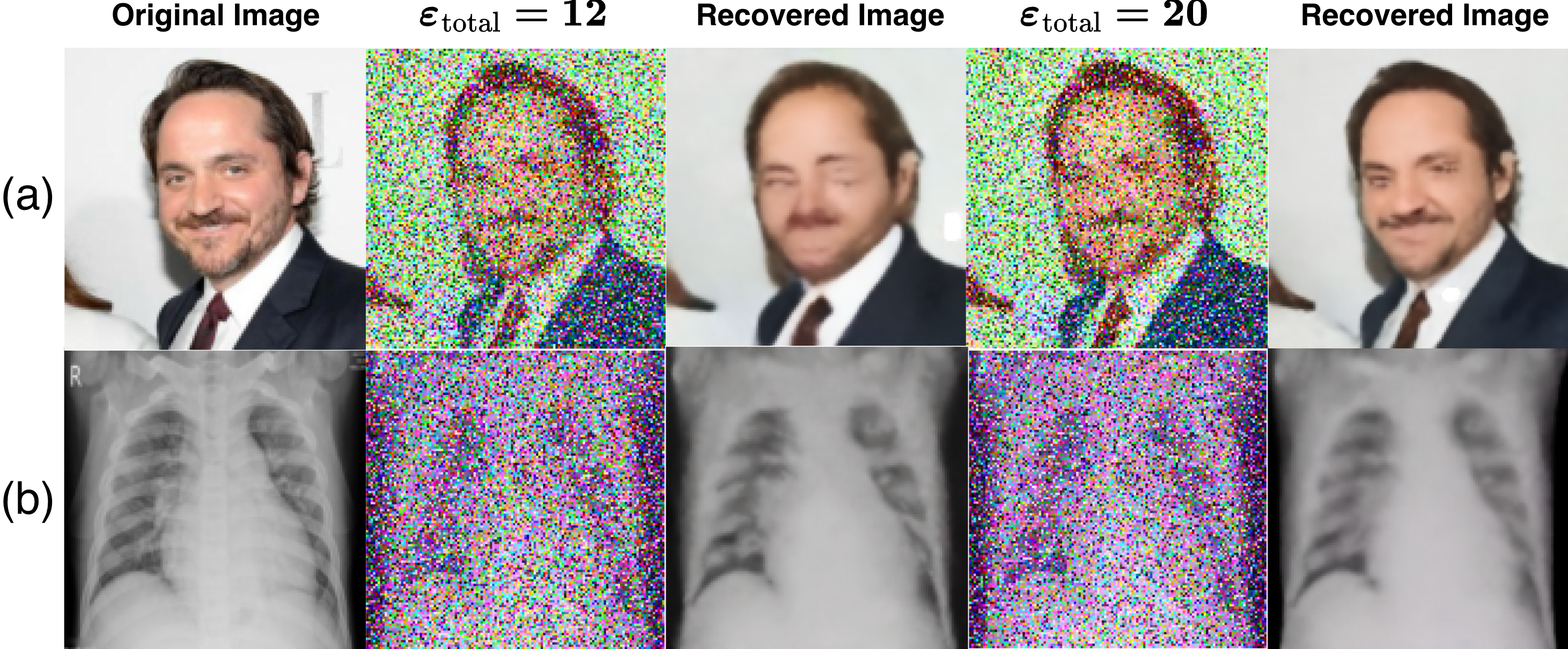}
  \caption{\textbf{Black-box reconstruction attack (w/o LL pruning) }}
  \label{fig:blacbox}
\end{figure}
\subsection{Comparison with block-level privacy} \label{sec:dct_comparison}
In the main paper, we show that LDP-Slicing outperforms DCTDP~\cite{dpdct} in face recognition benchmarks. To provide a fair comparison, we must convert the privacy guarantees at same level, as DCTDP defines privacy at the block-level, not pixel-level. Assume the standard configuration, DCTDP upsamples the facial image by 8 and
applies a Laplace mechanism with $\varepsilon_{\text{mean}}=0.5$ to each DCT coefficient within an $8\times8$ block. It also removes the DC coefficient for each color channel.  By the central Differential Private (DP) Sequential Composition Theorem~\cite{DworkMNS06}, the effective privacy budget bound for a single pixel is:
\[
\varepsilon_{\text{pixel}} \le ((8\times8)_{\text{block}} - 1_{\text{DC}})\times3 \times 0.5 = 94.5.
\]

In contrast, LDP-Slicing operates under a total budget of $\varepsilon_{\text{total}} = 20$. This derivation suggests that our method is approximately $4.7\times$ stricter than DCTDP. 

\subsection{Qualitative comparison against other SOTAs.}
To demonstrate the effectiveness against human observer, we add a visual comparison against other SOTAs in \cref{fig:Qualitativecmp}. Heuristic methods like PPFR-FD, DuetFace, and PartialFace retain sensitive facial features visible to the human observer. In contrast, DP/LDP methods conceal sensitive attributes in noise.

\section{Ethical considerations.}
Our primary experiment is trained on MS1Mv2 dataset \cite{guo2016msceleb1mdatasetbenchmarklargescale}. We acknowledge the ethical concerns about the MS1Mv2 dataset. Our usage is strictly limited to privacy defense research and adheres to CVPR’s ethical guidelines. We also recognize that our defense mechanism could theoretically be used by malicious actors to evade lawful systems (\eg, hiding illegal content). However, we believe the benefit of allowing end-users local control over data privacy outweighs these risks.  


\clearpage
{
    \small
    \bibliographystyle{ieeenat_fullname}
    \bibliography{main}
}


\end{document}